\documentclass[letterpaper, 10 pt, conference]{ieeeconf}

\usepackage{style}

\IEEEoverridecommandlockouts                         
\title{Statistical Convergence of Transformer Encoder-Accelerated Robust Reinforcement Learning}

\author{Suman Banerjee and Hiroyasu Tsukamoto\thanks{This work benefited from technical discussions within DARPA’s Safe and Assured Foundation Robots for Open eNvironments (SAFRON) Program, under contract number HR0011-25-3-0331. The authors are with the Department of Aerospace Engineering, The Grainger College of Engineering, University of Illinois Urbana-Champaign, Urbana, IL, USA (email: \texttt{sumanb2@illinois.edu, hiroyasu@illinois.edu})}}
\begin{document}

\maketitle
\thispagestyle{empty}
\pagestyle{empty}
\begin{abstract}
Obtaining the optimal action-value function in Markov decision processes is computationally intensive in large state--action spaces. In this study, we present statistically rigorous convergence results for a robust reinforcement learning algorithm warm-started by a transformer-based action-value function prediction, where natural language prompts encode task specifications. Our framework adopts the R-contamination model to characterize uncertainty in the state transition kernel, and employs conformal prediction to certify convergence via trajectory-level nonconformity scores constructed from the contracting Bellman residual. The resulting conformal quantile bounds the gap between the running and optimal action-value functions simultaneously over all iterations, thereby yielding a pre-certified stopping rule that requires little knowledge of the true transition kernel. Numerical case studies on perturbed maze environments of varying size and contamination level confirm that the transformer-based warm start measurably reduces the initial error and accelerates convergence, while the proposed conformal bounds track the true error trajectory more tightly than existing guarantees.

\end{abstract}

\vspace{-1.5mm}
\section{Introduction}
The framework of robust reinforcement learning (RL)~\cite{RL_SB,bertsekas1996neuro,NEURIPS2020_4eab60e5,Robust_RL_Rcont} faces the challenge of poor sample efficiency and slow convergence arising from poor initialization in large state--action spaces. Randomly initialized action-value functions must recover all structural information about their MDPs entirely through repeated trials, requiring an impractical amount of exploration followed by a series of Bellman updates before meaningful convergence~\cite{bellman_original,bertsekas1996neuro}.

Structured environments, however, are not featureless. Visual and natural language descriptions of environments and tasks may carry underlying semantic information of the tasks, which is usually discarded in nominal RL algorithms. One could naturally imagine that a well-informed initial prediction of the action-value function through learning-based encodings of such information would accelerate convergence. This work attempts to make that intuition precise by establishing theoretically rigorous, non-conservative bounds on the convergence rate that are both pre-certifiable and practically computable from finite data samples of contracting Bellman residuals. 
In particular, we fine-tune a pretrained transformer encoder~\cite{devlin2019bert} via supervised regression~\cite{hu2022lora} on action-value function tables generated offline, producing an initial estimate that encodes task-relevant structure before any environment interaction occurs.

\begin{figure}[h]
    \centering
    \includegraphics[width=0.37\textwidth]{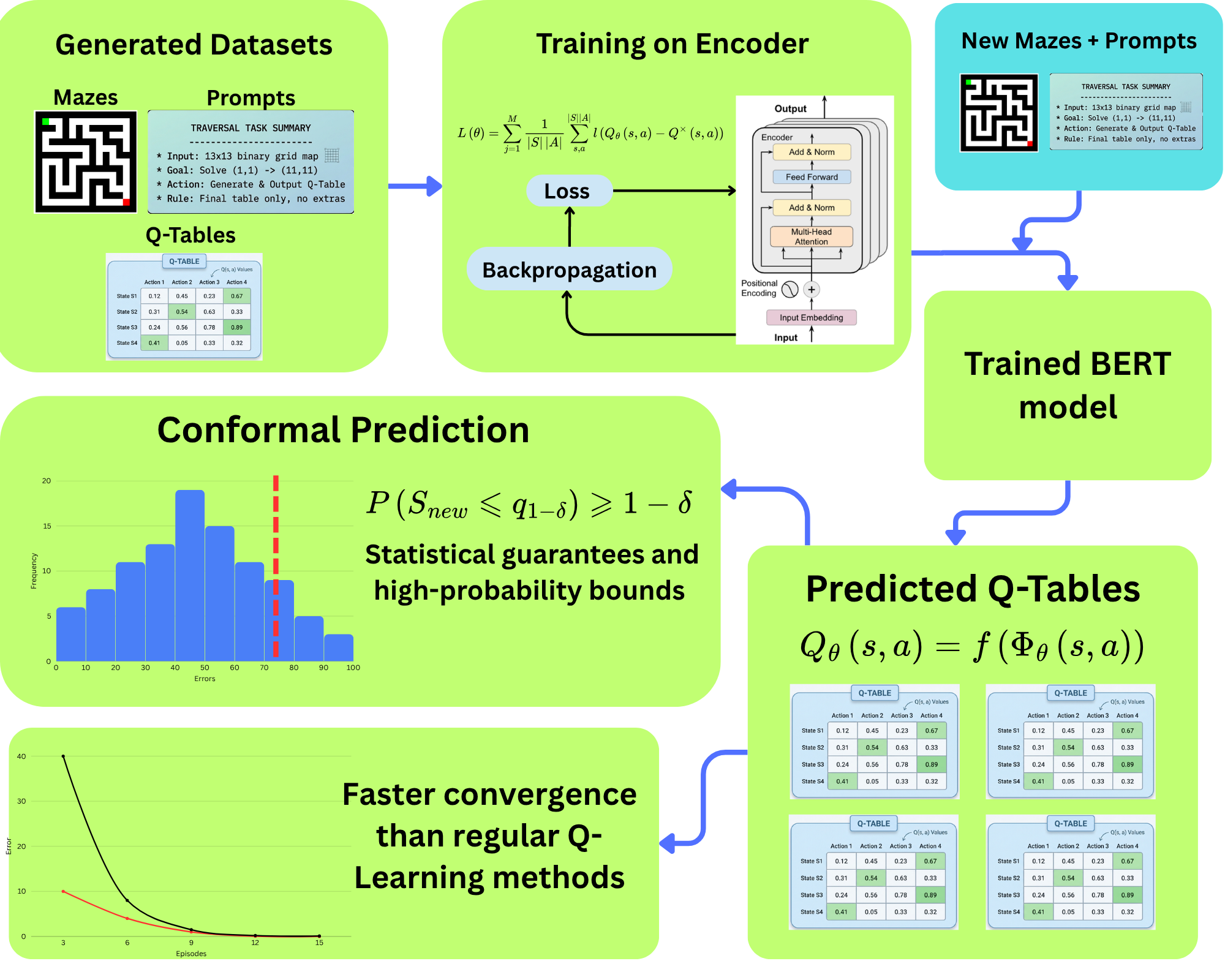}
    \vspace{-1mm}
    \caption{Overview of the proposed framework: using transformer encoder-based Q-function approximation with trajectory-wise conformal prediction to provide time-dependent and high-probability bounds on Q-Learning errors in $R-$contaminated RL environments.}
    \label{fig:overview}
    \vspace{-7mm}
    \setlength{\textfloatsep}{5pt}
    \setlength{\intextsep}{5pt}
\end{figure}
 Our framework considers the $R$-contamination model~\cite{Robust_RL_Rcont,Rcont} to characterize uncertainty in the state transition kernel, and employs conformal prediction to certify convergence via trajectory-level nonconformity scores constructed from the contracting Bellman residual. Our contributions include:
\begin{enumerate}
    \item We employ conformal prediction to statistically quantify the validity of transformer encoder-based approximations of the optimal action-value function, thus providing a formal tool to analyze the improvement of convergence behavior in the $R$-contaminated environments.
    \item We introduce a new trajectory-wise conformal prediction framework that yields tighter, finite-sample, distribution-free bounds on the gap between the running and optimal action-value functions, certifiable solely from observed Bellman residuals. 
    \item We demonstrate the effectiveness of these guarantees using bidirectional encoder representations from transformers (BERT) in perturbed maze environments of varying sizes and contamination levels.

\end{enumerate}

\subsubsection*{Related work}
Classical RL algorithms may assume that the environment used to train agents is identical to the environment in which the agent will be deployed~\cite{RL_SB}. However, this assumption is often violated in practice. Agents trained in simulation environments may encounter state transitions and rewards that differ from the ones of the real world, causing the learned policy to perform poorly upon deployment. This mismatch, commonly referred to as the sim-to-real gap, has motivated the study of robust Markov decision processes (RMDPs)~\cite{Uncertain_MDP, Robust_MDP}, where the true transition kernel is assumed to lie within an uncertainty set around a nominal model, defined via distance measures such as total variation, KL divergence, or the R-contamination model~\cite{Rcont}.
The goal of robust RL~\cite{Robust_RL_Rcont, Wang_Robust_PG, ghosh2025orvit, Robust_MARL} is to learn a policy that maximizes the return under worst-case transitions in the uncertainty sets. In~\cite{Robust_RL_Rcont}, ~\cite{Wang_Robust_PG}, the $R$-contamination model is used to derive Q-learning and policy gradient algorithms for the robust MDP, while~\cite{ghosh2025orvit} considers total variation and KL divergence uncertainty sets to propose distributionally robust value iteration. Other formulations treat model uncertainty as adversarial attacks~\cite{Robust_Adversarial_RL}, casting the robust RL as a two-player zero-sum discounted game. 

On the uncertainty quantification side, conformal prediction~\cite{conformal_original} provides a distribution-free, finite-sample statistical framework that produces coverage guarantees for arbitrary prediction models based on rank statistics of nonconformity scores, requiring only exchangeability of the data. The use of conformal prediction as a safety filter for RL in dynamic environments has been discussed in \cite{strawn2023conformal}, while \cite{gupta2023cammarl} applies conformal prediction on actions of multiple agents in RL to influence the action of a target agent. 
On the representation learning side, transformer encoder architectures have shown promise in RL settings due to their ability to capture long-range dependencies in sequential data~\cite{vaswani2017attention,voita2019analyzing,voita2019bottom}. Their self-attention mechanism allows encoding global context that reflects spatial, geometric, and semantic relationships across state representations. In~\cite{banino2021coberl}, the use of BERT to improve data efficiency using a novel combination of a gated transformer-XL~\cite{parisotto2020stabilizing} and long short-term memories~\cite{6795963} has been discussed, while in~\cite{RAJ2024100040}, transformer-based encodings are employed in language understanding tasks to detect discriminatory remarks across different languages, hence promoting safer online interactions. 
Since RL involves sequential decision-making with correlated, long-term trajectories, transformer architectures are well-suited for one-shot modeling of dependencies arising in such settings, capturing the structural patterns of the underlying MDP from task descriptions alone without requiring iterative environment interaction.
These results suggest that transformer encoders may offer useful inductive biases for structured RL tasks, though their theoretical integration with robust Q-learning convergence has not been previously addressed.

\\

\section{Preliminaries}
\label{sec:preliminaries}
\subsection{Robust Markov Decision Processes (RMDP)}
\label{RMDP}

Recall that an MDP is defined by a tuple $(\mathcal{S}, \mathcal{A}, \overline{\mathcal{P}}, r, \gamma)$, where $\mathcal{S}$ is the state space, $\mathcal{A}$ is the action space, $\overline{\mathcal{P}} = \{ \bar{P}(\cdot|s,a) \in \Delta_{|\mathcal{S}|} : s \in \mathcal{S}, a \in \mathcal{A} \}$ is the nominal transition kernel, $r(s,a): \mathcal{S}\times\mathcal{A}\to [0,1]$ is the reward function, and $\gamma \in [0,1)$ is the discount factor. Here, $\Delta_n$ denotes the $(n-1)$-dimensional probability simplex. To account for uncertainties in the state transition kernels, we consider a robust MDP~(RMDP)~\cite{Robust_MDP}. This is also defined by a tuple $(\mathcal{S}, \mathcal{A}, \mathcal{P}, r, \gamma)$, except that the true transition kernel is unknown and lies in the uncertainty set denoted by $\mathcal{P}$ given as $\mathcal{P}\coloneqq \bigotimes_{s \in \mathcal{S},\, a \in \mathcal{A}} \mathcal{U}(s,a)$
where $\mathcal{U}(s,a)$ is the uncertainty set of the transition kernel defined with respect to the state-action pair $(s,a)$ and $\bigotimes$ denotes the product measure. For a fixed policy $\pi$, the robust state value function for the RMDP is the discounted cumulative reward under the worst-case transitions, given by
$
\label{eq:robust-v}
V^\pi(s)
= \min_{P\in\mathcal{P}} \mathbb{E}_{\pi, P} [ \sum_{t=0}^\infty \gamma^t r(S_t, A_t)
\;|\; S_0 = s ]. 
$
Similarly, the robust action-value function is given by
$
Q^\pi(s,a)
= \min_{P\in\mathcal{P}} \mathbb{E}_{\pi, P} [ \sum_{t=0}^\infty \gamma^t r(S_t, A_t)
\;|\; S_0 = s, A_0 = a ].\nonumber
$
The goal of robust RL is to find the optimal robust policy $\pi^\star$ such that $\pi^\star \coloneqq \arg \max_\pi V^\pi(s)$. For notational simplicity, we denote $V^\star(s) \coloneqq V^{\pi^\star}(s)$ and $Q^\star(s,a) \coloneqq Q^{\pi^\star}(s,a)$.
The optimal robust value functions satisfy the robust Bellman equation, given by $V^\star(s) = \max_{a \in \mathcal{A}} Q^\star(s,a)$ and $Q^\star(s,a) = \mathcal{T}^\star Q^\star(s,a) \coloneqq r(s,a) + \gamma \sigma_{\mathcal{U}(s,a)}(V^\star)$, where $\sigma_{\mathcal{U}(s,a)}(V)\coloneqq \min_{P(\cdot|s,a)\in\mathcal{U}(s,a)} \mathbb{E}_{s'\sim P(\cdot|s,a)}V(s')$ denotes the worst-case expected next value. It is shown in~\cite{Robust_MDP} that the robust Bellman operator $\mathcal{T}^\star Q(s,a) \coloneqq r(s,a) + \gamma \sigma_{\mathcal{U}(s,a)}(V)$ is a contraction mapping and $Q^\star$ is its fixed point. Therefore, if the uncertainty set $\mathcal{P}$ is known exactly, then $Q^\star$ can be solved by dynamic programming. 
\subsection{$R$-Contamination Model and Robust Q-Learning}
\label{sec:robust_Q}
Since the exact form of $\mathcal{P}$ is rarely known in practice, we consider the $R$-contamination model~\cite{Robust_RL_Rcont,Rcont} to construct a tractable uncertainty set $\mathcal{P}$. First, let $\bar{P}(\cdot|s,a)$ be the nominal transition model, which generates sample trajectories for learning.
The $R$-contamination uncertainty set is defined as:
$
\label{eq:Rcont}
    \mathcal{U}(s,a) = \left\{ (1-R)\, \bar{P}(\cdot|s,a) + R\, q \;|\; q \in \Delta_{|\mathcal{S}|} \right\}
$ where $R \in [0,1]$ is a parameter that controls the size of the uncertainty. In general, a larger $R$ corresponds to larger uncertainty, leading to a more conservative optimal value function. In our problem setting, we assume that $R$ is known but $\bar{P}(\cdot|s,a)$ is \emph{unknown}, and thus we estimate $\mathcal{U}(s,a)$ online using maximum likelihood estimation (MLE) as follows.

Given an observed state transition $(s,a,s')$, the MLE estimate is $\widehat{P}(\cdot|s,a)\coloneqq \mathds{1}_{s'}$, which is an unbiased estimate of $\bar{P}(\cdot|s,a)$ because the expectation of the indicator function of an event is exactly the probability of that event. 
Then, the data-driven, estimated $R$-contaminated uncertainty set can be written as
$
\label{eq:Rcont_est}
    \widehat{\mathcal{U}}(s,a) = \big\{ (1-R)\mathds{1}_{s'} + R q \;|\; q \in \Delta_{|\mathcal{S}|} \big\}
$
which yields the following estimated worst-case next value: 
$
\label{eq:Rcont_next}
    \sigma_{\widehat{\mathcal{U}}(s,a)}(V) \coloneqq R \min_{s''\in \mathcal{S}} V(s'') + (1-R)V(s').
$ For $R$-contamination models, the regular Q-learning, which does not account for the uncertainty in the transition kernel, can yield a learned optimal policy that degrades significantly in the true environment under such uncertainty. The robust Q-learning algorithm was proposed by~\cite{Robust_RL_Rcont} to address this issue by explicitly incorporating the $R$-contamination uncertainty set into Q-learning. 
It is implemented in an online episodic fashion, using multiple trajectory data samples collected under some exploratory policy $\pi_b(\cdot|s)$ (e.g., $\epsilon$-greedy policy). With the robust MDP introduced earlier, the temporal-difference update of robust Q-learning can be written as $Q_{t+1}(s,a) = (1-\alpha_t) Q_t(s,a)+ \alpha_t (r + \gamma \sigma_{\widehat{\mathcal{U}}(s,a)}\left(\textstyle\max_{a'\in\mathcal{A}} Q_t(s,a') \right) )$, where $\alpha_t$ denotes the learning rate and $t$ is the iteration index. Therefore, by definition of $\sigma_{\widehat{\mathcal{U}}(s,a)}$, this becomes
\begin{align}
    Q_{t+1}(s,a) &= (1-\alpha_t) Q_t(s,a) + \alpha_t (r + \gamma R\textstyle \min_{s''\in \mathcal{S}} V_t(s'') \nonumber \\
&+ \gamma(1-R)V_t(s') )
\label{eq:Q_update_Rcont}
\end{align}
%
As can be seen, robust Q-learning is a model-free, online, and off-policy RL algorithm, which has the following convergence property. 
\begin{lemma} [Robust Q-learning~\cite{bertsekas1996neuro,Robust_RL_Rcont,NEURIPS2020_4eab60e5}]
\label{th:Q_converge}
If a constant learning rate $\alpha_t = \alpha$ is used in~\eqref{eq:Q_update_Rcont}, then for any $\phi\in(0,1)$ and $\epsilon\in\left(0, \tfrac{1}{1-\gamma}\right)$, there exist constants $\hat{c}$ and $c_0$ such that the following bound holds for all $t\leq T$ at least with probability $1-6\phi$:
$
    \Delta Q_t
\leq
\underbrace{\tfrac{(1-\rho)^{k}}{1-\gamma}\Delta Q_0}_{\text{contracting term}}
+
\underbrace{\tfrac{5\hat{c}\gamma}{1-\gamma}
\sqrt{\alpha\log\left(\tfrac{|\mathcal{S}||\mathcal{A}|T}{\phi}\right)}
+\epsilon}_{\text{steady-state error term}}
$
as long as the learning rate satisfies $0<\alpha \log\left(\tfrac{|\mathcal{S}||\mathcal{A}|T}{\phi}\right)<1$ and the time horizon satisfies
$
T \geq \tfrac{c_0}{\mu_{\min}} \Big(
\tfrac{1}{(1-\gamma)^5 \epsilon^2} + \tfrac{t_{\mathrm{mix}}}{1-\gamma}
\Big)
\log\left(\tfrac{T|\mathcal{S}||\mathcal{A}|}{\phi}\right)\log\left(\tfrac{1}{\epsilon(1-\gamma)^2}\right).
$
The notation is summarized in Table~\ref{tab:finite_time_terms}.
\end{lemma}


\begin{table}[t]
\centering
\caption{Definitions of quantities in Lemma~\ref{th:Q_converge}}
\label{tab:finite_time_terms}
\renewcommand{\arraystretch}{1.15}
\scriptsize
\begin{tabular}{l l}
\hline
Quantity & Definition / meaning \\
\hline
$\Delta Q_t$ & $\coloneqq \|Q_t - Q^\star\|_\infty$ \\
$k$ & $\coloneqq \max\Big\{0,\left\lfloor\tfrac{t-t_{\mathrm{th}}}{t_{{\mathrm{frame}}}}\right\rfloor \Big\}$ \\
$\rho$ & $\coloneqq (1 - \gamma)\left(1 - (1 - \alpha)^{\mu_{\mathrm{frame}}}\right)$ \\
$\mu_{\mathrm{frame}}$ & $\coloneqq \tfrac{1}{2}\mu_{\min} t_{\mathrm{frame}}$ \\
$t_{\mathrm{th}}$ & $\coloneqq \max\{
\tfrac{2\log\left(\tfrac{1}{(1-\gamma)^2\epsilon}\right)}{\alpha\mu_{\min}},
\, t_{\mathrm{frame}}
\}$ \\

$\mu_{\min}$ & Minimum state-action visitation frequency under the sampling process \\

$t_{\mathrm{mix}}$ & Mixing time of the Markov chain induced by the behavior policy \\

$t_{\mathrm{frame}}$ & Frame length used in the blockwise finite-time analysis \\

\hline
\vspace{-9mm}
\end{tabular}
\end{table}

\subsection{Transformer Natural Language Encoders}
Transformer encoders are neural architectures designed to parse structured input data into a sequence of continuous representations using stacked feed-forward and multi-head attention layers~\cite{vaswani2017attention}. This architecture enables them to simultaneously learn short- and long-range contextual relationships of input data. When applied to natural language~\cite{transformer_nlp}, these encoders produce latent representations that distill the meaning of each token in relation to the context of the full input sequence. In this study, we fine-tune the BERT (Bidirectional Encoder Representations from Transformers)~\cite{devlin2019bert} model to predict the optimal action-value function obtained through the robust Q-learning of Sec.~\ref{sec:robust_Q}. Specifically, the model is trained to take as input a natural language prompt describing the reinforcement learning task specification and output the corresponding action-value function.

\subsection{Conformal Prediction}
Conformal prediction is a distribution-free framework for generating finite-sample coverage guarantees for arbitrary uncertainty predictors based on the rank statistics of nonconformity scores. If applied to a regression model $Y = \hat{f}(X)$, such scores may be defined as $s_i=\|\hat{f}(X_i)-Y_i\|,\forall i \in \{1,2,\cdots,N\}$, where $Y_i$ denotes the ground-truth output samples corresponding to the input samples $X_i$. For a prescribed failure probability $\delta \in (0,1)$, we define the conformal quantile $q_{1-\delta}$ as the $\lceil(1-\delta)(N+1)\rceil/N$-th smallest nonconformity score in $\{s_i\}_{i\in\{1,2,..,N\}}\cup\infty$, where $\lceil(\cdot)\rceil$ denotes the ceiling function.
\begin{lemma}[Conformal Prediction~\cite{conformal_original}]
\label{CP}
If, for any new sample $(X_{i_{new}},Y_{i_{new}})$, all the nonconformity scores in $\{s_i\}_{i \in \{1,2,..,N\}\cup\{i_{new}\}}$ are exchangeable, then a statistical bound on the prediction error may be postulated as follows:
$
    \mathbb{P}(\|\hat{f}(X_{i_{new}})-Y_{i_{new}}\|\leq q_{1-\delta})\geq 1-\delta
$
where $\mathbb{P}$ is the probability measure induced by the joint distribution of the data $\{(X_i,Y_i)\}_{i=1}^N$ and the new sample $(X_{i_{\mathrm{new}}},Y_{i_{\mathrm{new}}})$.
\end{lemma}



\section{Methods}

\subsection{Robust $Q$ Function Prediction from Natural Language}
Directly predicting the action-value function $Q$ in high-dimensional state spaces, particularly when the task is also represented by high-dimensional inputs, can be complex and could produce poor generalization. 
Raw information on states and task specifications often contains irrelevant features that may distract RL algorithms during extraction~\cite{zhang2020learning}, leading to inefficient training.
In this paper, we consider reinforcement learning tasks specified in natural language and hence introduce a transformer encoder $\Phi_{\theta}$, which first converts states $s$, actions $a$, and natural-language task specifications $\mathscr{L}$ into numerical token embeddings, and then encodes them into latent representations capturing task-relevant features. Here, $\theta$ denotes the parameters of the transformer. In particular, by selecting the output features to be the predicted optimal action-value function, $Q_\theta(s,a;\mathscr{L})$, we define $\Phi_{\theta}$ as:
\begin{align}
\label{eq_def_Qtheta}
    Q_\theta(s,a;\mathscr{L}) = (\psi_{\theta}\circ\varphi_{\theta})(s,a,\mathscr{L}), 
    \Phi_{\theta} \coloneqq \psi_{\theta}\circ\varphi_{\theta}
\end{align}
where $\varphi_{\theta}(s,a,\mathscr{L})\in\mathbb{R}^L$ yields the numerical embedding of length $L$, and $\psi_{\theta}$ denotes the encoder. Given $M$ data samples of the optimal action-value function $Q^*(s,a;\mathscr{L}_j)$ corresponding to sampled task specifications $\{\mathscr{L}_j\}^M_{j=1}$, the transformer $\Phi_{\theta}$ is trained by supervised regression using the following Huber loss over all state and action pairs:
$\mathcal{L}(\theta) \coloneqq \sum_{j=1}^M\tfrac{1}{|\mathcal{S}||\mathcal{A}|} \sum_{s,a} \ell(Q_\theta(s,a;\mathscr{L}_j) - Q^*(s,a;\mathscr{L}_j))$,
where $\ell(\Delta) \coloneqq \tfrac{1}{2}\Delta^2$ for $|\Delta|\le 1$ and $\ell(\Delta) \coloneqq |\Delta|-\tfrac{1}{2}$ otherwise.
Instead of training $\Phi_{\theta}$ from scratch over all tunable parameters in $\theta$, we initialize $\Phi_{\theta}$ with a pre-trained BERT~\cite{devlin2019bert} backbone and fine-tune selected parameters using LoRA~\cite{hu2022lora}. Since the embedding function $\varphi_{\theta}$ is often defined by linguistic rules and/or fixed structural representations (e.g., numerical indices associated with language tokens), the statistical properties of $ Q_\theta(s,a;\mathscr{L})$ for sampled task specifications can be characterized through $\varphi_{\theta}(s,a,\mathscr{L})\in\mathbb{R}^L$, yielding the following result.
\begin{lemma}
\label{lemma_transformer_cp}
Suppose that the parameter $\theta$ of the transformer~$\Phi_{\theta}$ is fixed after training, and we have $N$ data samples of natural-language task specifications, $\{\mathscr{L}_i\}^N_{i=1}$, not used for training. If, for any new sample $\mathscr{L}_{i_{new}}$, all the numerical embeddings in $\{\varphi_{\theta}(s,a,\mathscr{L}_i)\}_{i \in \{1,2,..,N\}\cup\{i_{new}\}}$ are exchangeable over all state and action pairs, then we have the following bound:
\begin{align}
    \label{eq_transformer_cp}
    \mathbb{P}\left\{\|Q_\theta(\mathscr{L}_{i_{new}}) - Q^*(\mathscr{L}_{i_{new}})\|_{\infty} \leq q_{1-\delta_0}\right\} \geq 1-\delta_0
\end{align}
where $\delta_0 \in (0,1)$ is the failure probability, $q_{1-\delta_0}$ is the $\lceil(1-\delta_0)(N+1)\rceil/N$-th smallest nonconformity score in $\{s_i\}_{i\in\{1,2,..,N\}}\cup\infty$ with $s_i \coloneqq \|Q_\theta(\mathscr{L}_{i}) - Q^*(\mathscr{L}_{i})\|_{\infty}$, and $\|(\cdot)\|_{\infty}$ denotes the infinity norm defined as
$\|Q_\theta(\mathscr{L}_{i}) - Q^*(\mathscr{L}_{i})\|_{\infty} \coloneqq \sup_{(s,a)\in\mathcal{S}\times\mathcal{A}}\|Q_\theta(s,a;\mathscr{L}_{i}) - Q^*(s,a;\mathscr{L}_{i})\|$
for the predicted/optimal action-value functions $Q_\theta$ and $Q^*$.
\end{lemma}
\begin{proof}
This follows from Lemma~\ref{CP}.
\end{proof}
\vspace{-2mm}

\subsection{Transformer-Accelerated Robust Q-learning}
\vspace{-3mm}
Although Lemma~\ref{lemma_transformer_cp} provides a theoretical error bound on the action-value function prediction, the quantile $q_{1-\delta}$ in~\eqref{eq_transformer_cp} is almost always nonzero due to the empirical nature of LoRA fine-tuning. We first derive an explicit dependence of the convergence behavior of robust Q-learning on the initialization induced by the transformer prediction $Q_{\theta}$, showing accelerated convergence to $Q^*$.

\begin{theorem} [Robust Q-Learning Warm-Started by $Q_{\theta}$]
\label{theorem_warmstart}
Suppose we are given a new task specification $\mathscr{L}_{i_{new}}$ with all the assumptions in Lemma~\ref{lemma_transformer_cp} satisfied. If we initialize the action-value function as $Q_0(s,a;\mathscr{L}_{i_{new}}) = Q_{\theta}(s,a;\mathscr{L}_{i_{new}})$, we have the following bound for all $t\leq T$ and some $\delta_0 \in (0,1)$ at least with probability $1-6\phi-\delta_0$: \begin{equation}
\begin{aligned}
\label{eq_speedup}
\resizebox{\linewidth}{!}{$\displaystyle
    \Delta Q_\infty(\mathscr{L}_{i_{new}})
    \leq
    \underbrace{\tfrac{(1-\rho)^{k}q_{1-\delta_0}}{1-\gamma}}_{\text{contracting term}}
    +
    \underbrace{\tfrac{5\hat c\gamma}{1-\gamma}
    \sqrt{\alpha\log\left(\tfrac{|\mathcal{S}||\mathcal{A}|T}{\phi}\right)}
    +\epsilon}_{\text{steady-state error term}}
$}
\end{aligned}
\end{equation}
where $\Delta Q_\infty(\mathscr{L}_{i_{new}}) \coloneqq \|Q_t(\mathscr{L}_{i_{new}}) - Q^*(\mathscr{L}_{i_{new}})\|_{\infty}$ with $Q_0(s,a;\mathscr{L}_{i_{new}}) = Q_{\theta}(s,a;\mathscr{L}_{i_{new}})$, and the other notations are consistent with the ones in Lemmas~\ref{th:Q_converge} and~\ref{lemma_transformer_cp}. Also, if we select the learning rate $\alpha$ to be
$$\alpha = \tfrac{c_1}{\log\left(\tfrac{T|\mathcal{S}||\mathcal{A}|}{\phi}\right)} \min\left(\tfrac{1}{t_{\mathrm{mix}}},\tfrac{\epsilon^2(1-\gamma)^4}{\gamma^2}\right)$$
then we have $\Delta Q_\infty(\mathscr{L}_{i_{new}}) \leq 3\epsilon$ whenever $k \geq \tfrac{\log\left(\tfrac{q_{1-\delta_0}}{(1-\gamma)\epsilon}\right)}{\log\left(\tfrac{1}{1-\rho}\right)}$ for a suitable constant $c_1$.
\end{theorem}
\begin{proof}
The relation~\eqref{eq_speedup} follows from Lemmas~\ref{th:Q_converge} and~\ref{lemma_transformer_cp}. The condition on $k$ is by bounding the first term of~\eqref{eq_speedup} by $\epsilon$. The constant $c_1$ can be found as in~\cite{NEURIPS2020_4eab60e5}.
\end{proof}

While Theorem~\ref{theorem_warmstart} is useful from an analytical standpoint, it is of limited practical value because the steady-state error term can dominate the upper bound of~\eqref{eq_speedup}, making it overly conservative. Empirically, however, as shown in Sec.~\ref{sec_numerical}, the actual convergence of robust Q-learning is substantially faster.
We derive the following tighter bound to bridge this gap at the expense of a reduced probability guarantee.

\begin{theorem}[Statistical Convergence]
\label{thm:time_uniform_conformal_q}
Suppose that we are given a collection of task specifications $\{\mathscr{L}_{i}\}_{i=1}^N$ with exchangeable embeddings as in Lemma~\ref{lemma_transformer_cp}. Suppose we are given a new task specification $\mathscr{L}_{i_\mathrm{new}}$ with all the assumptions in Lemma~\ref{lemma_transformer_cp} satisfied. For each~\( i \), let \( \{Q_t(s,a;\mathscr{L}_{i})\}_{t=0}^T \) denote the sequence of Q-functions generated by the robust Q-learning initialized as $Q_0(s,a;\mathscr{L}_{i}) =  Q_\theta(s,a;\mathscr{L}_{i})$ with the transformer prediction $Q_\theta$ defined in~\eqref{eq_def_Qtheta}. Also, let \( Q^*(s,a;\mathscr{L}_{i}) \) denote the corresponding optimal action-value function.
Inspired by the result~\eqref{eq_speedup} of Theorem~\ref{theorem_warmstart}, we define the following step-wise nonconformity scores:
\begin{align}
\label{eq_def_nonconformity}
s_{t,i}
= \max\left\{0,
\Delta Q_t(\mathscr{L}_{i}) -\tfrac{\Delta Q_0(\mathscr{L}_{i}) (1-\rho)^k}{1-\gamma}
\right\}
\end{align}
where $\Delta Q_t(\mathscr{L}_{i}) \coloneqq \|Q_t(\mathscr{L}_{i}) - Q^*(\mathscr{L}_{i})\|_\infty$ and $\Delta Q_0(\mathscr{L}_{i}) \coloneqq \|Q_\theta(\mathscr{L}_{i}) - Q^*(\mathscr{L}_{i})\|_\infty$.
Then, for a task specification \( \mathscr{L}_{i_\mathrm{new}} \) whose embedding is exchangeable with those of $\{\mathscr{L}_{i}\}_{i=1}^N$, the following holds:
\vspace{-2mm}
\begin{equation}
\label{eq:probability_bound}
\resizebox{\linewidth}{!}{$\displaystyle
\mathbb{P}\Big(
\forall t \in I,~\Delta Q_t(\mathscr{L}_{i_\mathrm{new}})
\le
\tfrac{\Delta Q_0(\mathscr{L}_{i_\mathrm{new}}) (1-\rho)^k}{1-\gamma}
+
q_{t,1-\delta_t}
\Big)
\geq
1 - \sum_{t=1}^T \delta_t
$}
\end{equation}
where $I \coloneqq \{1,\cdots,T\}$ and $q_{t,1-\delta_t}$ is the $(1-\delta_t)$-th empirical quantile of the $t$-th nonconformity scores $\{s_{t,i}\}_{i=1}^N$. Also, for any $\epsilon \in (0,\infty)$, we have $\Delta Q_t(\mathscr{L}_{i_\mathrm{new}}) \leq \epsilon + q_{t,1-\delta_t}$ at least with probability $1 - \sum_{t=1}^T \delta_t -\delta_0$ for the defined $k$, where $\delta_0$ and $q_{1-\delta}$ are again given in~\eqref{eq_transformer_cp} of Lemma~\ref{lemma_transformer_cp}.
\end{theorem}
\begin{proof}
For each fixed \( t \), Lemma~\ref{CP} implies
$
\mathbb{P}\left(
s_{t,i_\mathrm{new}} \le q_{t,1-\delta_t}
\right)
\ge 1 - \delta_t
$
where $s_{t,i_\mathrm{new}}$ is the new nonconformity score defined for the new sample $\mathscr{L}_{i_\mathrm{new}}$. For the failure event given as $E_t \coloneqq \left\{ s_{t,i_\mathrm{new}} > q_{t,1-\delta_t} \right\}$ so that \( \mathbb{P}(E_t) \le \delta_t \) yields
$
\mathbb{P}\left( \bigcup_{t=1}^T E_t \right) \le \sum_{t=1}^T \delta_t
$
by Boole's inequality. Taking the probability of the complement of $\bigcup_{t=1}^T E_t$ gives~\eqref{eq:probability_bound}. The rest follows from Theorem~\ref{theorem_warmstart}.
\end{proof}

Theorem~\ref{thm:time_uniform_conformal_q} provides a time-uniform, distribution-free and high-probability bound on the convergence rate of pre-trained model-predicted Q-functions. This is in contrast with the standard classical approach to Q-learning, where asymptotic and expectation-based bounds are prescribed. In addition, the theorem provides a finite-sample, non-asymptotic bound which holds for every iteration during the learning process by defining how the $t-$th step error reduces under a warm start. The uncertainty in the process is represented by data-driven quantiles. This yields a method to certify that the learning dynamics of a transformer-predicted warm start Q-function are reliable and hence, useful in unknown, safety-critical and data-sparse environments. A critical limitation in practically applying the bounds~\eqref{eq_speedup} and~\eqref{eq:probability_bound} is the sim-to-real gap. Given a task specification~$\mathcal{L}$, let $Q_t(s,a;\mathcal{L},\mathcal{E}_{\mathrm{sim}})$ denote the action value function constructed in a simulated environment $\mathscr{E}_{\mathrm{sim}}$ and $Q^*(s,a;\mathcal{L},\mathcal{E}_{\mathrm{real}})$ be the corresponding optimal action-value function in the real world environment $\mathcal{E}_{\mathrm{real}}$. At a high level, the discrepancy $\Delta Q_t(\mathcal{L},\mathcal{E}_{\mathrm{sim}},\mathcal{E}_{\mathrm{real}}) \coloneqq \|Q_t(\mathcal{L},\mathcal{E}_{\mathrm{sim}})-Q^*(\mathcal{L},\mathcal{E}_{\mathrm{real}})\|_{\infty}$ can be computed as follows~\cite{Robust_RL_Rcont,NEURIPS2020_4eab60e5}:
\begin{align}
&\Delta Q_t(\mathcal{L},\mathcal{E}_{\mathrm{sim}},\mathcal{E}_{\mathrm{real}})
\leq \Delta P_t(\mathcal{L},\mathcal{E}_{\mathrm{sim}})
+\Delta E_t(\mathcal{L},\mathscr{E}_{\mathrm{sim}},\mathscr{E}_{\mathrm{real}})\nonumber\\
&+(\|I-\Lambda_t\|_\infty+\gamma\|\Lambda_t\|_\infty) \Delta Q_{t-1}(\mathcal{L},\mathcal{E}_{\mathrm{sim}},\mathcal{E}_{\mathrm{real}})
\label{eq:psi_decomp}
\end{align}
where $\Lambda_t \in \mathbb{R}^{|\mathcal{S}||\mathcal{A}| \times |\mathcal{S}||\mathcal{A}|}$ is a diagonal matrix encoding the step size $\alpha_t$ at the visited state--action pair $(s_t, a_t)$ and zero elsewhere, $P \in \mathbb{R}^{|\mathcal{S}||\mathcal{A}| \times |\mathcal{S}|}$ denotes the
sample transition matrix from the single observed transition at time $t$, $\Delta P_t(\mathcal{L},\mathcal{E}_{\mathrm{sim}})$ denotes the bounded error arising from the discrepancy between the MLE estimated transition model and the true simulated model at $t$, and $\Delta E_t(\mathcal{L},\mathscr{E}_{\mathrm{sim}}, \mathscr{E}_{\mathrm{real}})$ captures the cumulative discrepancy between the simulated and real environments. Inspired by recent work on closed-loop statistical inference~\cite{hsu2025statistical}, we finally analyze the robustness of the warm-started robust Q-learning algorithm by applying the idea analogous to Theorem~\ref{thm:time_uniform_conformal_q} to the model~\eqref{eq:psi_decomp}.
\begin{theorem}[Robustness of Robust Q-Learning]
Suppose that we are given a collection of task specifications and environments, $\{\mathscr{L}_{i}\}_{i=1}^N$, $\{\mathscr{E}_{i,\mathrm{sim}}\}_{i=1}^N$,  and $\{\mathscr{E}_{i,\mathrm{real}}\}_{i=1}^N$, with exchangeable embeddings as in Lemma~\ref{lemma_transformer_cp}. Suppose we are given a new task specification and environments, $\mathscr{L}_{i_\mathrm{new}}$, $\mathscr{E}_{i_\mathrm{new},\mathrm{sim}}$, and $\mathscr{E}_{i_\mathrm{new},\mathrm{real}}$, with all the assumptions in Lemma~\ref{lemma_transformer_cp} satisfied.
Define the following trajectory-wise nonconformity scores:
$
\label{eq_def_nonconformity_traj}
s_{i,\mathrm{traj}}
= \sup_{t\in[t_0,T]}\Delta P_t(\mathcal{L}_{i},\mathcal{E}_{i,\mathrm{sim}})+\Delta E_t(\mathcal{L}_{i},\mathscr{E}_{i,\mathrm{sim}},\mathscr{E}_{i,\mathrm{real}})~~
$
for $\Delta P_t$ and $\Delta E_t$ given in~\eqref{eq:psi_decomp}, where $t_0$ denotes the time when the relation~\eqref{eq:psi_decomp} becomes valid with $\|I-\Lambda_t\|_\infty < 1$~\cite{Robust_RL_Rcont,NEURIPS2020_4eab60e5}. For each~\( i \), let \( \{Q_t(s,a;\mathscr{L}_{i},\mathscr{E}_{i,\mathrm{sim}})\}_{t=0}^T \) denote the sequence of Q-functions generated by the robust Q-learning initialized as $Q_{t_0}(s,a;\mathscr{L}_{i}) = Q_\theta(s,a;\mathscr{L}_{i},\mathscr{E}_{i,\mathrm{sim}})$ with the transformer prediction $Q_\theta$ defined in~\eqref{eq_def_Qtheta}. Also, let \( Q^*(s,a;\mathscr{L}_{i},\mathscr{E}_{i,\mathrm{real}}) \) denote the corresponding optimal action-value function. Then, for a new task specification and environments, $\mathscr{L}_{i_\mathrm{new}}$, $\mathscr{E}_{i_\mathrm{new},\mathrm{sim}}$, and $\mathscr{E}_{i_\mathrm{new},\mathrm{real}}$, whose embeddings are exchangeable with their respective calibration samples, the following holds:
\begin{align}
&\mathbb{P}(
\forall t \in [t_0,T],~\Delta Q_t(\mathscr{L}_{i_\mathrm{new}},\mathscr{E}_{i_\mathrm{new},\mathrm{sim}},\mathscr{E}_{i_\mathrm{new},\mathrm{real}}) \\
&\leq \Delta Q_{t_0}        (\mathscr{L}_{i_\mathrm{new}},\mathscr{E}_{i_\mathrm{new},\mathrm{sim}},\mathscr{E}_{i_\mathrm{new},\mathrm{real}}) (1-\lambda)^{t-t_0}\\
&+\tfrac{q_{1-\delta_{\mathrm{traj}}}}{\lambda}(1-(1-\lambda)^{t-t_0})) \geq 1-\delta_{\mathrm{traj}}
\end{align}
where $\lambda \coloneqq \inf_{t\in[t_0,T]}(1-\|I - \Lambda_t\|_{\infty}-\gamma\|\Lambda_t\|_{\infty})$, and $q_{1-\delta_{\mathrm{traj}}}$ is the $(1-\delta_{\mathrm{traj}})$-th empirical quantile of the trajectory-wise nonconformity scores $\{s_{i,\mathrm{traj}}\}_{i=1}^N$. Furthermore, for any $\epsilon \in (0,\infty)$, we have $$\Delta Q_t(\mathscr{L}_{i_\mathrm{new}},\mathscr{E}_{i_\mathrm{new},\mathrm{sim}},\mathscr{E}_{i_\mathrm{new},\mathrm{real}}) \leq \epsilon + \tfrac{q_{1-\delta_{\mathrm{traj}}}(1-(1-\lambda)^{t-t_0})}{\lambda}$$ at least with probability $1 - \delta_{\mathrm{traj}} -\delta_0$ whenever $\lambda \in (0,1)$ and $t$ satisfies the following:
$t \geq \log(q_{1-\delta_0}/\epsilon)\big/[-\log(1-\lambda)]$
where $\delta_0$ and $q_{1-\delta_0}$ are again given in~\eqref{eq_transformer_cp} of Lemma~\ref{lemma_transformer_cp}.
\end{theorem}
\begin{proof}
Using~\eqref{eq:psi_decomp} iteratively, we get
$
    \Delta Q_t 
    \leq \tilde{\lambda}^{t-t_0}\Delta Q_{t_0}
    + \sum_{\tau=t_0+1}^{t}\tilde{\lambda}^{t-\tau}(\Delta P_\tau + \Delta E_\tau) 
$
where $\tilde{\lambda} \coloneqq 1-\lambda$ and the arguments are omitted for notational simplicity. Using the result of Lemma~\ref{CP}, with probability at least $1-\delta_{\mathrm{traj}}$, we have $\|\Delta P_\tau + \Delta E_\tau\| \leq q_{1-\delta_{\mathrm{traj}}}$ for the new task and the environments, which implies
$
    \Delta Q_t 
    \le\tilde{\lambda}^{t-t_0}\Delta Q_{t_0} 
    +  q_{1-\delta_{\mathrm{traj}}}\cdot\frac{1-\tilde{\lambda}^{t-t_0}}{\lambda}
$
establishing $s_{i,\mathrm{traj}}$ and the bound on $t$.
\end{proof}



\section{Numerical Case Studies}

\label{sec_numerical}
We perform numerical simulations on randomly generated $n\times n$ mazes (Fig.~\ref{fig:maze_examples}) and train BERT\cite{devlin2019bert} on vision--language prompt-based datasets. To facilitate the environment construction, we add a boundary layer around the original mazes. For simulations, we used an NVIDIA L40S GPU with 48GB of GDDR6 memory. Each dataset contained 30,000 sample prompts with responses, and the data generation and training took about 20 hours on average for each maze size.
We use a green square to denote the starting square and the red one to denote the goal (Fig. \ref{fig:maze_examples}). $Q^*$ is now obtained using the hyperparameters given in Table \ref{table:hyper_Q} and the prompt is next constructed in two parts, the \textit{Instruction} and the \textit{Response}. A sample \textit{Instruction} is given in Fig. \ref{fig:prompt}.

\begin{figure}[h]
    \centering
    \includegraphics[width=0.5\textwidth, trim=1.5cm 2.5cm 1cm 1.5cm, clip]{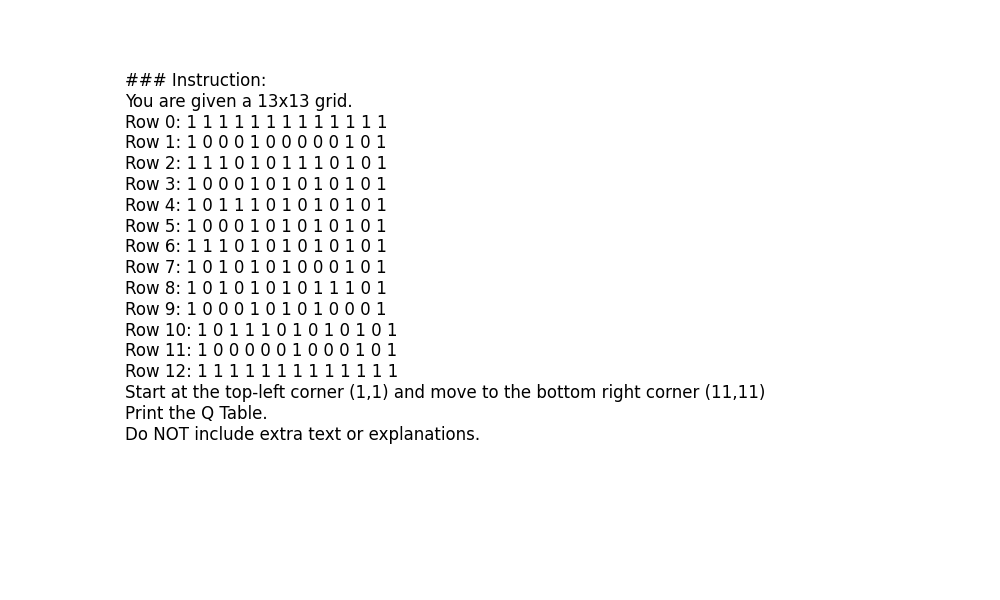}
    \vspace{-10mm}
    \caption{An example \textit{Instruction} for a $13\times 13$ maze}
    \label{fig:prompt}
\end{figure}

\begin{figure}[h]
    \centering

    \begin{subfigure}[b]{0.15\linewidth}
        \centering
        \includegraphics[width=\linewidth, trim=1cm 1cm 1cm 1cm, clip]{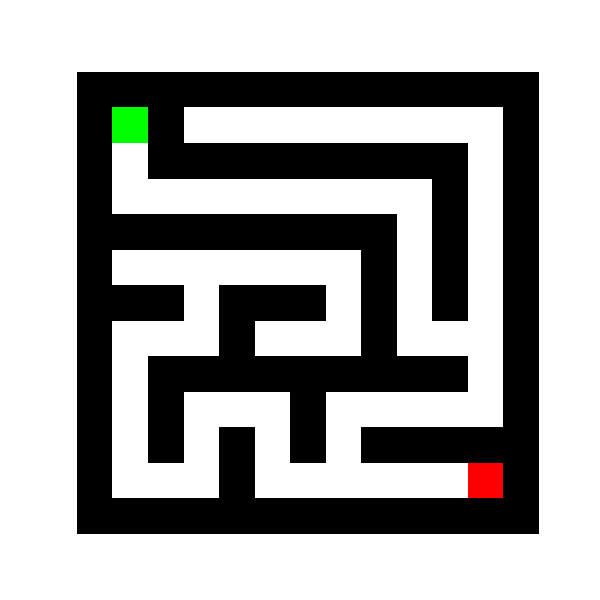}
        \caption*{(a) $t=0$ (start)}
    \end{subfigure}
    \hfill
    \begin{subfigure}[b]{0.15\linewidth}
        \centering
        \includegraphics[width=\linewidth, trim=1cm 1cm 1cm 1cm, clip]{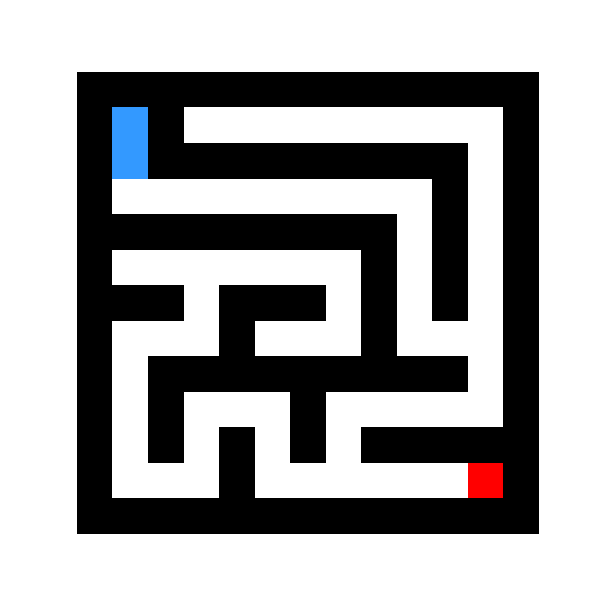}
        \caption*{(b) $t=1$ (w/ BERT)}
    \end{subfigure}
    \hfill
    \begin{subfigure}[b]{0.15\linewidth}
        \centering
        \includegraphics[width=\linewidth, trim=1cm 1cm 1cm 1cm, clip]{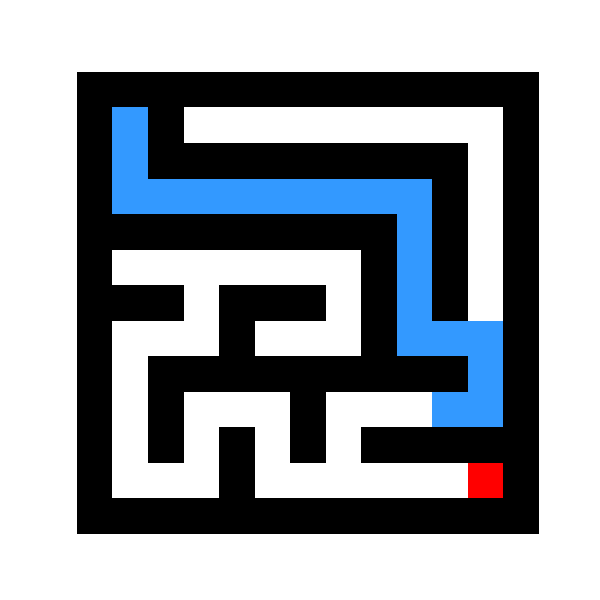}
        \caption*{(c) $t=2$ (w/ BERT)}
    \end{subfigure}
    \hfill
    \begin{subfigure}[b]{0.15\linewidth}
        \centering
        \includegraphics[width=\linewidth, trim=1cm 1cm 1cm 1cm, clip]{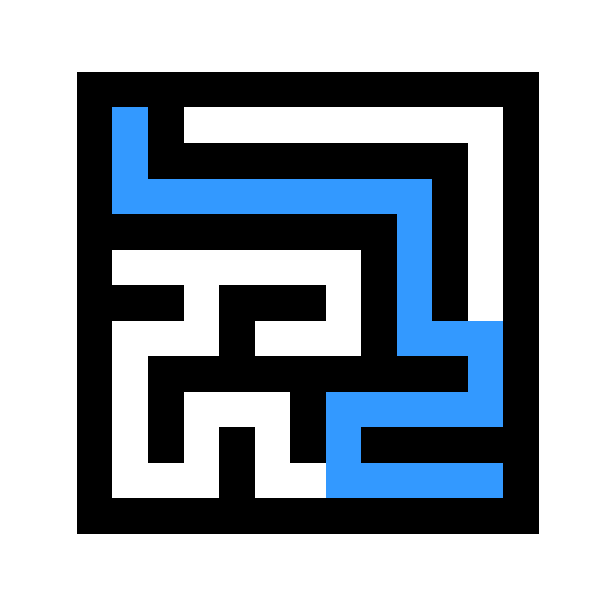}
        \caption*{(d) $t=3$ (w/ BERT)}
    \end{subfigure}
    \hfill
    \begin{subfigure}[b]{0.15\linewidth}
        \centering
        \includegraphics[width=\linewidth, trim=1cm 1cm 1cm 1cm, clip]{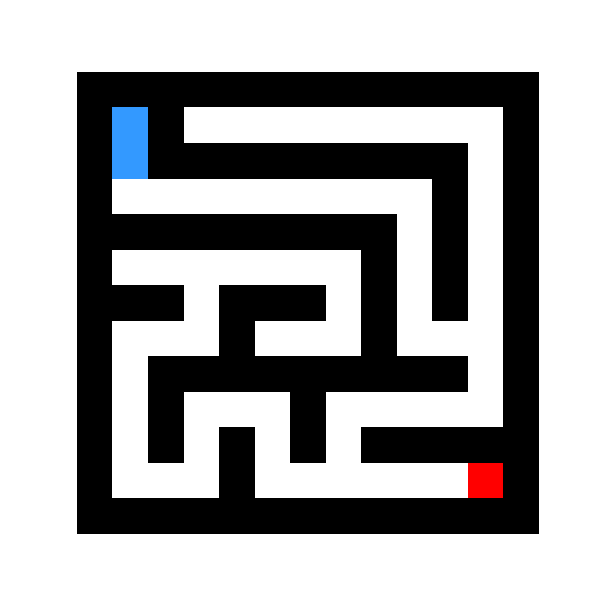}
        \caption*{(e) $t=15$ (w/o BERT)}
    \end{subfigure}
    \vspace{-1.5mm}
    \caption{
    Maze-solving progress for a $13\times13$ maze with $R=0.5$: (a) initial maze; (b)--(d) BERT-based warm start after 1--3 episodes; (e) standard Q-learning after 15 episodes.
    }
    \label{fig:maze_examples}
    \vspace{-5mm}
\end{figure}
The \textit{Response} corresponds to $string(Q^*)$. We then construct a dataset of these \textit{Instruction-Response} pairs and train the model on it, keeping the $R-$Contamination fixed at $\{0.2,0.4,0.5\}$ for $9\times9, 11\times11$ and $13\times13$ maze sizes, respectively, to approximate the optimal action-value function. For this procedure, we sample 2000 random mazes and compute their corresponding error trajectories. During evaluation, $R$ is sampled randomly from $[0.1,0.9]$ to assess the robustness of our framework under varying uncertainty. The hyperparameters and the corresponding reward policy for Q-learning are given in Table \ref{table:hyper_Q}.

\begin{remark}
The rewards in Table~\ref{table:hyper_Q} satisfy the $r\in[0,1]$ assumption in Section~\ref{RMDP} after an affine normalization, $\tilde r=(r-r_{\min})/(r_{\max}-r_{\min})$, which preserves Q-value ordering and the optimal policy.
\vspace{-2mm}
\end{remark}

\subsection{Accelerated Convergence under Warm-start Initialization}
The trained model generates an initial estimate of $Q^*$, denoted by $Q_\theta=Q_{\text{BERT}}$ during evaluation, and we initialize the learning process with this estimate, i.e.,
$
Q_0(s,a) = Q_{\text{BERT}}(s,a),
$
and compare it against the standard initialization:
$
Q_0(s,a) = 0, \quad \forall (s,a) \in S \times A.
$ For both initialization schemes, we track the error with respect to the optimal Q-function $Q^*$ over each iteration
$
\|Q_t - Q^*\|_\infty,
$ distinguishing between $\|Q_t - Q^*\|_{\text{BERT}}$ and $\|Q_t - Q^*\|_{\text{true}}$, the error trajectories initialized with $Q_{\text{BERT}}$ and the zero function, respectively. Fig. \ref{fig:convergence} shows the effect of learned initialization on the convergence behavior of $||Q_t - Q^*||_\infty$, where the blue curves represent the true error under the BERT-based warm start and the red curves correspond to the standard Q-Learning algorithm. Under the influence of the BERT-based warm start, the true error rapidly converges to zero under varying levels of uncertainty, illustrating the impact of learned initialization on convergence behavior.

\vspace{-2mm}
\begin{figure}[htbp]
    \centering

    \begin{subfigure}[b]{0.31\linewidth}
        \centering
        \includegraphics[width=\linewidth, trim=0cm 0cm 1cm 1cm, clip]{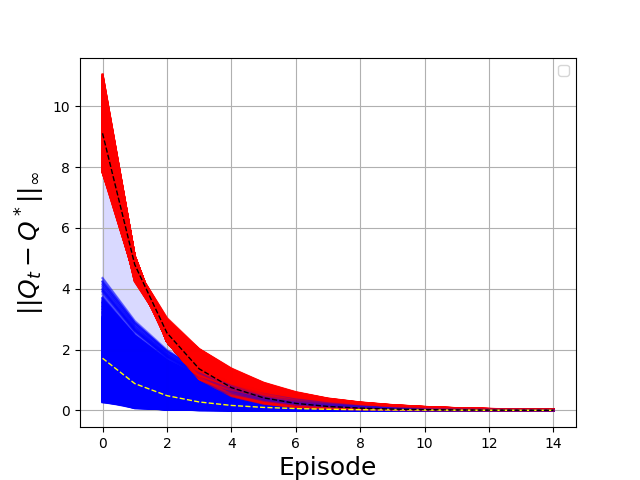}
        \caption{}
        \label{fig:step_b1}
    \end{subfigure}
    \hfill
    \begin{subfigure}[b]{0.31\linewidth}
        \centering
        \includegraphics[width=\linewidth, trim=0cm 0cm 1cm 1cm, clip]{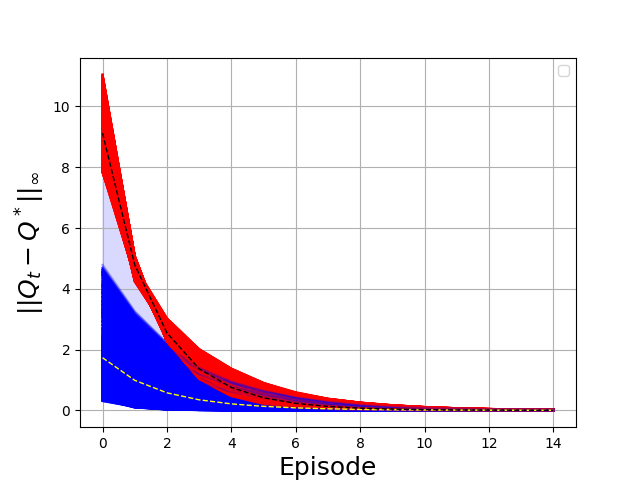}
        \caption{}
        \label{fig:step_c1}
    \end{subfigure}
    \hfill
    \begin{subfigure}[b]{0.31\linewidth}
        \centering
        \includegraphics[width=\linewidth, trim=1cm 1cm 1cm 1cm, clip]{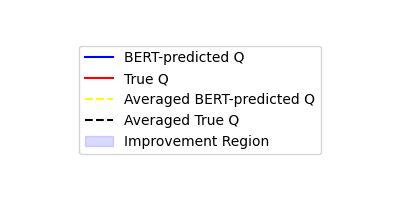}
        \caption*{}
        \label{fig:step_d1}
    \end{subfigure}
     \vspace{-2.5mm}
    \caption{Improvement in convergence behavior using warm start for $R \in [0,1]$ in mazes of size: (a) 11$\times$11 and (b) 13$\times$13.}
    \label{fig:convergence}
    
\end{figure}
\begin{figure}[htbp]
\vspace{1mm}
    \centering

    \begin{subfigure}[b]{0.31\linewidth}
        \centering
        \includegraphics[width=\linewidth, trim=0.5cm 0cm 1cm 0.5cm, clip]{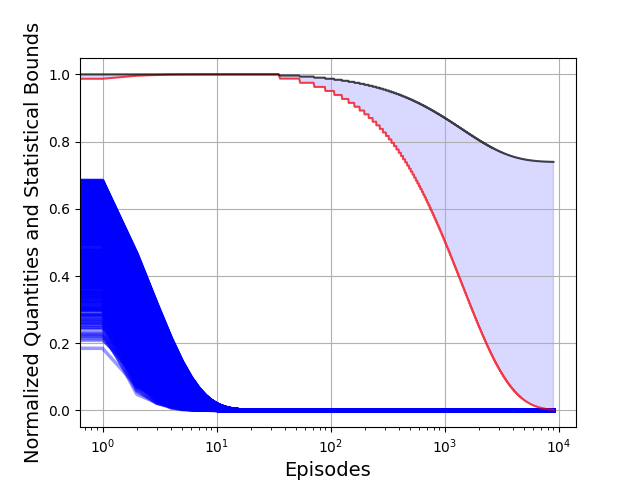}
        \caption{}
        \label{fig:step_a}
    \end{subfigure}
    \hfill
    \begin{subfigure}[b]{0.31\linewidth}
        \centering
        \includegraphics[width=\linewidth, trim=0.5cm 0cm 1cm .5cm, clip]{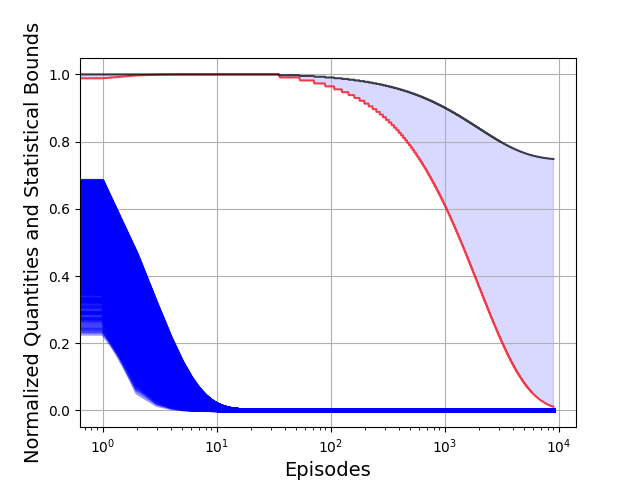}
        \caption{}
        \label{fig:step_c}
    \end{subfigure}
    \hfill
    \begin{subfigure}[b]{0.31\linewidth}
        \centering
        \includegraphics[width=\linewidth, trim=0.5cm 0cm 1cm 1cm, clip]{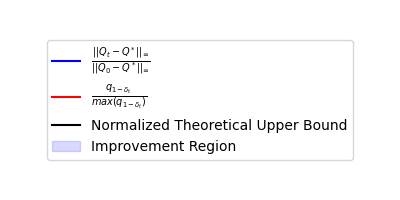}
        \caption*{}
        \label{fig:step_d}
    \end{subfigure}
    \vspace{-2.5mm}
    \caption{Demonstration of the proposed stricter bounds in mazes of size: (a) 11$\times$11 and (b) 13$\times$13.}
    \label{fig:stricter_bound}
    \vspace{-5mm}
\end{figure}

\subsection{Statistical Refinement of Bounds and Conformal Prediction}
Fig. \ref{fig:stricter_bound} shows that the proposed framework from Theorems \ref{theorem_warmstart} and \ref{thm:time_uniform_conformal_q} tightens previous theoretical guarantees in \cite{Robust_RL_Rcont} and the proposed bounds follow the true error. For each maze sample $\{\mathscr{L}_{i}\}_{i=1}^N$, we compute $\{s_{t,i}\}_{t=1}^T$, and the corresponding $(1-\delta_t)$-th empirical quantiles over all episodes. A time-dependent conformal correction is now constructed, which upper bounds the nonconformity scores with high probability, along with the $(1-\delta)$-th quantile for the initial prediction error $||Q_{BERT}-Q^*||$, to calculate the contracting and the steady-state error term for every sample. Episodes are plotted on the logarithmic $x$-axis, the blue curves represent the normalized true error trajectory $\frac{\|Q_t - Q^*\|_\infty}{\max_t \|Q_t - Q^*\|_\infty}$, the red curves represent the normalized conformal correction term $\frac{q_{1-\delta_t}}{\max_t q_{1-\delta_t}}$, and the black curves represent the normalized theoretical upper bound $\frac{(1-\rho)^k q_{1-\delta}+5\hat{c}\gamma\sqrt{\alpha\log(|\mathcal{S}||\mathcal{A}|T/\phi)}+(1-\gamma)\epsilon}{\max_t[(1-\rho)^k q_{1-\delta}+5\hat{c}\gamma\sqrt{\alpha\log(|\mathcal{S}||\mathcal{A}|T/\phi)}+(1-\gamma)\epsilon]}$ respectively. Fig. \ref{fig:histogram} shows a steady monotonic decrease in the logarithmic nonconformity score frequencies which suggests a heavy-tailed nature with the 90\textsuperscript{th}-percentile threshold denoted. Most scores lie to the left of this bound, although a heavy tail suggests the presence of relatively rarer larger deviations, which is accounted for by Thm. \ref{thm:time_uniform_conformal_q}. The threshold shifts across different maze sizes, indicating the framework's response to problem complexity, but the bulk of the distribution remains concentrated to small errors, successfully demonstrating system efficacy.
\begin{figure}[htbp]
\vspace{-3.5mm}
    \centering

    \begin{subfigure}[b]{0.4\linewidth}
        \centering
        \includegraphics[width=\linewidth,
        trim=0.5cm 0cm 0.5cm 1cm, clip]
        {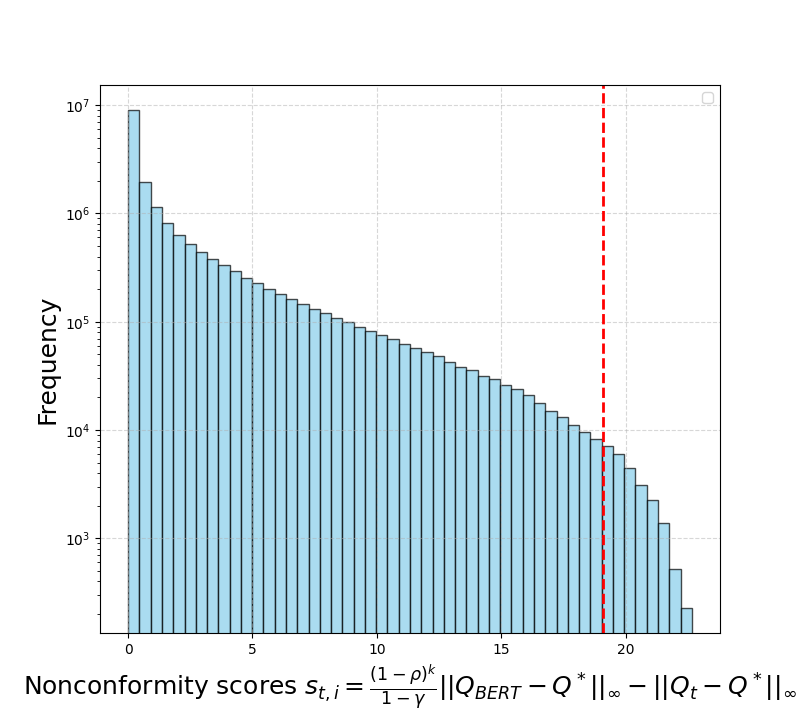}
        \caption{$11\times11$ maze}
        \label{fig:histogram_11x11}
    \end{subfigure}
    \hfill
    \begin{subfigure}[b]{0.4\linewidth}
        \centering
        \includegraphics[width=\linewidth,
        trim=0.5cm 0cm 0.5cm 1cm, clip]
        {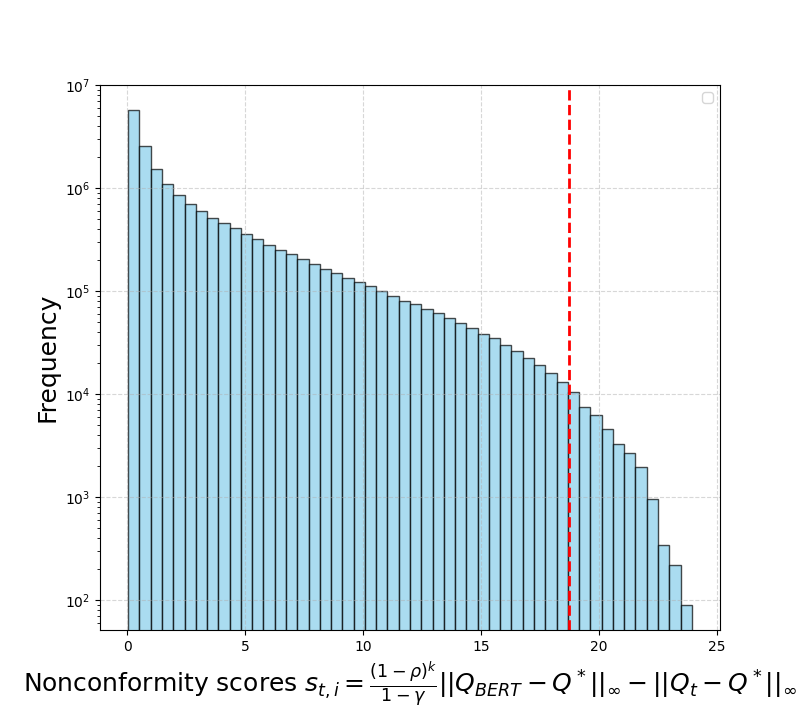}
        \caption{$13\times13$ maze}
        \label{fig:histogram_13x13}
    \end{subfigure}

    \caption{Empirical distributions of $s_{t,i}$ for the BERT-based
    warm start in mazes of sizes $11\times11$ and $13\times13$, with
    the 90\textsuperscript{th}-percentile shown.}
    \label{fig:histogram}
    \vspace{-7mm}
\end{figure}
\begin{table}[h!]
\centering
\small
\setlength{\tabcolsep}{4pt}
\renewcommand{\arraystretch}{0.95}
\begin{tabular}{lc|lr}
\hline
\textbf{Parameter} & \textbf{Value} & \textbf{Reward Policy} & \textbf{Reward} \\
\hline
$SIZE$   & $\{9,11,13\}$       & Free Space       & $-0.1$   \\
$\gamma$ & $0.7$              & Obstacle         & $-1.0$   \\
$M$      & $9000$             & Goal             & $+5.0$   \\
$\alpha$ & $0.9$              & Outside Boundary & $-100.0$ \\
$R$      & $\{0.2,0.4,0.5\}$  &                  &          \\
\hline
\end{tabular}
\caption{Hyperparameters and reward policy used in robust Q-learning for numerical simulations}
\label{table:hyper_Q}
\vspace{-7mm}
\end{table}

\section{Conclusion}
This study presents a new framework to combine transformer encoder-based representation learning with robust Q-learning methods, where a BERT-based warm start seeds the algorithm with a strong initial prediction that reduces initial error and improves sample efficiency. Conformal prediction yields tighter, finite-sample, distribution-free Q-error bounds uniformly over all iterations than existing theoretical guarantees. The numerical simulations validate these claims, demonstrating that the convergence behavior improves with the BERT-based warm start, and that the conformal bounds consistently track the true error more closely than prior results.
\vspace{-2mm}

\bibliographystyle{IEEEtran}
\bibliography{refs}

\end{document}